\documentclass{article}

\usepackage[dblblindworkshop,final]{neurips_2026}

\usepackage[utf8]{inputenc} 
\usepackage[T1]{fontenc}    
\usepackage{hyperref}       
\usepackage{url}            
\usepackage{booktabs}       
\usepackage{amsfonts}       
\usepackage{nicefrac}       
\usepackage{microtype}      
\usepackage{xcolor}         
\usepackage{subcaption}
\usepackage{graphicx}
\usepackage{multirow}
\usepackage{amsmath, amssymb, amsthm}
\usepackage{bm}
\usepackage{mathtools}
\newtheorem{theorem}{Theorem}

\newtheorem*{assumption*}{Assumption}

\usepackage{thmtools}
\usepackage{thm-restate}
\usepackage{cleveref}
\crefname{theorem}{Theorem}{Theorems}
\crefname{corollary}{Corollary}{Corollaries}
\crefname{lemma}{Lemma}{Lemmas}
\crefname{definition}{Definition}{Definitions}
\crefname{assumption}{Assumption}{Assumptions}
\crefname{figure}{Figure}{Figures}
\crefname{section}{Section}{Sections}
\crefname{appendix}{Appendix}{Appendices}

\title{Port-Hamiltonian Neural Networks for Systems with Multiple Asymptotically Stable Equilibria}
\workshoptitle{
AXIOM - Foundations of Efficient Deep Learning}

\author{%
  Simon Heilig$^{1}$\thanks{Equal contribution.} \quad
  Jens P\"uttschneider$^{2}$\footnotemark[1] \quad
  Mohammad Itani$^{2}$ \\
  \textbf{Asja Fischer$^{1}$} \quad
  \textbf{Timm Faulwasser$^{2}$} \\[4pt]
  $^{1}$Faculty of Computer Science, Ruhr University Bochum, Germany \\
  $^{2}$Institute of Control Systems, Hamburg University of Technology, Germany \\[2pt]
  \texttt{simon.heilig@rub.de}, \texttt{jens.puettschneider@tuhh.de}
}
\begin{document}

\maketitle

\begin{abstract}
  Stable port-Hamiltonian neural networks certify asymptotic stability by construction. 
  Yet, their Hamiltonian is a global Lyapunov function with a single global minimum, so they can represent only dynamic systems with {one} attractor. We demonstrate that this excludes even simple systems with energy landscapes forming a double well, and we overcome the restriction by parametrising the Hamiltonian as a {product} of Bregman divergences generated by one input-convex network. 
  We prove that the resulting model is locally Lyapunov stable, that the coexistence of stable equilibria forces additional non-asymptotically-stable equilibria to exist, that all equilibria lie in a bounded region, and under a hyperbolicity assumption that almost-everywhere stability holds.
  On three systems our approach is able to recover the energy surface characteristics and improve the convergence speed by 1.8$\times$-8.5$\times$.
\end{abstract}

\section{Introduction}

Neural ODEs~\citep{chen2018neural} and their physics-informed variants fit a vector field to observed trajectories, restricting the hypothesis class to what is known about the physics: Hamiltonian networks conserve energy~\citep{greydanus2019hamiltonian}, and port-Hamiltonian neural networks (PHNNs)~\citep{desai2021port,neary2023compositional,cherifi2026nonlinear} add dissipation and external ports \citep{vanderschaft2014port}.
Such a bias gives a {certificate}: the energy dissipation holds across the parameters, so stability guarantees follow from the architecture, and it buys {efficiency}, since a hypothesis class matching the data-generating process needs fewer samples and fewer optimiser steps \citep{drgovna2025safe}.

Stability guarantees for learned dynamics have moved from correction to construction: \cite{kolter2019learning} project a freely learned vector field onto the halfspace certified by an input-convex Lyapunov function with the fully input-convex network (FICNN) \citep{amos2017input}. Instead, \cite{yu2021onsagernet} rely on a generalized Onsager decomposition of the dissipative and conservative parts, and \cite{roth2025stable} make the Hamiltonian convex, forming a stable port-Hamiltonian network (s-PHNN) allowing for a strong statement about the neural network's behavior: an s-PHNN is globally asymptotically stable for {every} parameter value, naturally transferring to inference time.

Many dissipative systems of interest are not globally asymptotically stable with respect to one equilibrium: a buckled beam settles into one of two deflected shapes and a pendulum over magnets onto one of the magnets.
These systems have finitely many isolated attractors separated by saddles, with basins partitioning state space up to a null set. \cite{breunung2025learning} uses Koopman representations to model such systems, leaving their model without rigorous stability guarantees.
This is complementary to Lyapunov-based region-of-attraction certification \citep{richards2018lyapunov} and to projection-based construction of asymptotically stable invariant sets \citep{takeishi2021learning}.
\vspace{-4pt}
\paragraph{Contributions.}
(i) We draw upon the Duffing oscillator~\citep{haddad2008nonlinear} to illustrate that s-PHNNs cannot represent a system with multiple stable equilibria dynamics due to its global asymptotic stability guarantee.
(ii) We propose an extension via a multiplicative Hamiltonian, where each factor captures one stable equilibrium, and evaluate it on three systems with multiple stable equilibria.
(iii) We provide theoretical guarantees for the local asymptotic stability of the minimum-energy equilibria, the existence and boundedness of other equilibria, and almost-global asymptotic stability of the stable equilibria.

\section{Mismatch in Inductive Biases}
\label{sec:counterexample}
A simple example of a port-Hamiltonian system with multiple asymptotically stable equilibria is the Duffing oscillator.
The system states are position $q$ and momentum $p$, its total energy is described by the Hamiltonian $\mathrm{H}(q,p) = \tfrac12 p^2 + \tfrac14(q^2-1)^2$, and its system dynamics are $\dot{\bm{x}} = (\bm{J}-\bm{R})\nabla_{\bm{x}} \mathrm{H}$, where $\bm{J}$ is the canonical skew-symmetric interconnection matrix and $\bm{R} = \operatorname{diag}(0,\mu)$ is the damping matrix, with, e.g., $\mu = 0.1$.
This system possesses two stable equilibria at $\bar{\bm{x}}_{1,2} = 
[\pm1, 0]^T$ corresponding to global minima of the energy $H$, while the {saddle} $\bm{x}_s=[0,0]^T$ is another equilibrium whose stable manifold divides the two {basins} and the convergence depends on the initial condition $\bm{x}_0$.

In \Cref{fig:energy_landscape_sphnn}, we show the trained energy landscape of the s-PHNN from~\cite{roth2025stable}, illustrating how the s-PHNN can only learn one of the two stable equilibria of the Duffing oscillator.
This is precisely what one would expect from the inductive bias imposed by the global asymptotic stability guarantee of the s-PHNN: it can have only one attractor and hence cannot capture the bistable nature of the Duffing oscillator, instead it learns only one of its two stable equilibria.
This calls for a model which preserves the port-Hamiltonian inductive bias and its stability guarantees while allowing for multiple stable equilibria.

\begin{figure}[t]
  \centering
  \makebox[\linewidth][l]{\hspace*{-0.02\linewidth}%
    \includegraphics[width=0.75\linewidth,  trim={1.35cm 0 0 0},clip]{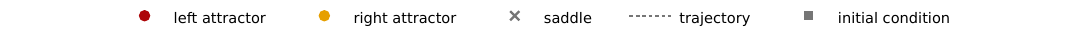}}%
  \\[-1pt]
  \begin{minipage}[b]{0.0712\linewidth}%
    \centering
    \includegraphics[height=3cm]{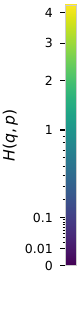}%
    \subcaption*{\phantom{(a)}}%
  \end{minipage}%
  \begin{subfigure}[b]{0.285\linewidth}%
    \includegraphics[width=\linewidth,height=3cm]{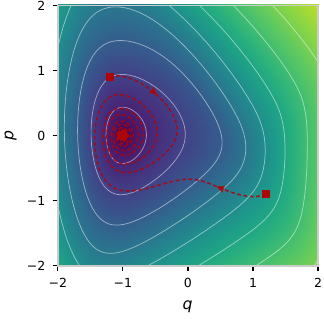}%
    
    \subcaption{s-PHNN}\label{fig:energy_landscape_sphnn}%
  \end{subfigure}%
  \begin{subfigure}[b]{0.285\linewidth}%
    \includegraphics[width=\linewidth,height=3cm]{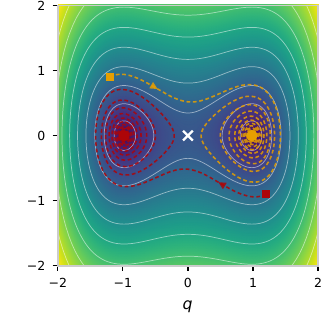}%
    \subcaption{ms-PHNN (ours)}\label{fig:energy_landscape_msphnn}%
  \end{subfigure}%
  \hspace*{1em}
  \begin{subfigure}[b]{0.33\linewidth}%
  \centering
    \includegraphics[width=\linewidth,height=3.0cm]{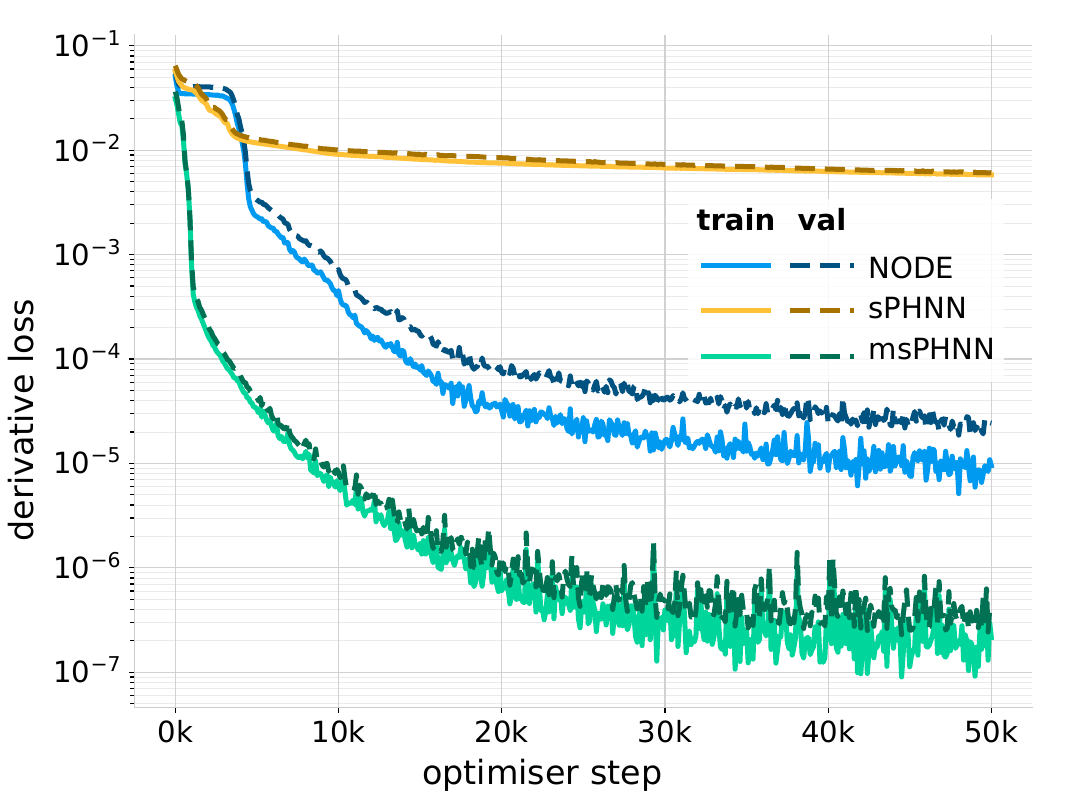}%
    \subcaption{Training \& validation loss}\label{fig:curves_dw}%
  \end{subfigure}%
  \caption{
	Energy landscape and two trajectories of the Duffing oscillator for (a) the fitted s-PHNN and (b) our ms-PHNN, showing that the ms-PHNN captures both asymptotically stable equilibria, whereas the s-PHNN learns only one.
  	(c) Training and validation derivative loss for neural ODE (NODE), s-PHNN, and our ms-PHNN, median over 10 seeds.
  }%
  \label{fig:sphnn_failure}
\end{figure}

\section{Port-Hamiltonian Neural Networks with Multiple Stable Equilibria}
To this end, we propose a port-Hamiltonian Neural Network with multiple stable equilibria (\mbox{ms-PHNN}) based on a multiplicative Hamiltonian.
The ms-PHNN dynamics with state $\bm{x}\in\mathbb{R}^{n_x}$ and parameters $\bm{\theta}\in\mathbb{R}^{n_\theta}$ is
\begin{equation}
	\label{eq:phnn}
    \dot{\bm{x}} = \mathrm{f}_{\mathrm{PH}} (\bm{x}) = \left[\bm{J}_{\bm{\theta}}(\bm{x}) - \bm{R}_{\bm{\theta}}(\bm{x})\right]\nabla_{\bm{x}} \mathrm{H}_{\bm{\theta}}(\bm{x}), \quad \bm{x}(0) = \bm{x}_0,
    \quad \bm{J}_{\bm{\theta}} = -\bm{J}_{\bm{\theta}}^\top, \quad \bm{R}_{\bm{\theta}} = \bm{R}_{\bm{\theta}}^\top \succeq \bm{0}.
\end{equation}
For our ms-PHNN we parametrize the Hamiltonian as the product of $m$ factors, one per equilibrium from the set of pairwise distinct points $\bar{\mathbb{X}} = \{\bar{\bm{x}}_1,\dots,\bar{\bm{x}}_m\}$
\begin{equation}
	\label{eq:hamiltonian}
    \mathrm{H}_{\bm{\theta}}(\bm{x}) = \prod_{k=1}^m D_{\bm{\theta},k}(\bm{x}),
    \qquad
    D_{\bm{\theta},k}(\bm{x}) = f_{\bm{\theta},k}(\bm{x}) - f_{\bm{\theta},k}(\bar{\bm{x}}_k) - \nabla_{\bm{x}} f_{\bm{\theta},k}(\bar{\bm{x}}_k)^{\!\top}(\bm{x} - \bar{\bm{x}}_k).
\end{equation}
Each $D_{\bm{\theta},k}$ follows the proposed Bregman divergence construction for a single equilibrium~\citep{roth2025stable}, and all $m$ factors are generated by a \emph{single} FICNN $f_{\bm{\theta}} \in C^\infty(\mathbb{R}^{n_x},\mathbb{R}^m)$ whose $k$-th output produces the $k$-th factor.
Throughout, we take the interconnection and dissipation matrices to be continuously differentiable, $\bm{J}_{\bm{\theta}}, \bm{R}_{\bm{\theta}} \in C^1(\mathbb{R}^{n_x}, \mathbb{R}^{n_x \times n_x})$.

Since every Bregman divergence is nonnegative and $D_{\bm{\theta},k}(\bar{\bm{x}}_k) = 0$, we have $\mathrm{H}_{\bm{\theta}}(\bm{x}) \geq 0$ for all $\bm{x}$, with $\mathrm{H}_{\bm{\theta}}(\bar{\bm{x}}) = 0$ at the minimum-energy equilibria $\bar{\bm{x}} \in \bar{\mathbb{X}}$.
By using one shared FICNN, we avoid using a separate network for each minimum-energy equilibrium: the factors $D_{\bm{\theta},k}(\bm{x})$ share all hidden layers and differ only in the final convex read-out.
Following the setup of \cite{roth2025stable}, each minimum-energy equilibrium $\bar{\bm{x}}_k$ may be fixed a-priori or learned jointly with the parameters of the FICNN.
The flow of the ms-PHNN~\eqref{eq:phnn},\eqref{eq:hamiltonian}, starting from the initial condition $\bm{x}_0 \in \mathbb{R}^{n_x}$ is given by~
$
	\bm{x}(t; \bm{x}_0)
	=
	\bm{x}_0
	+
	\int_0^t
	\left[\bm{J}_{\bm{\theta}}(\bm{x}(\tau; \bm{x}_0)) - \bm{R}_{\bm{\theta}}(\bm{x}(\tau; \bm{x}_0))\right]
	\nabla_{\bm{x}} \mathrm{H}_{\bm{\theta}}(\bm{x}(\tau; \bm{x}_0))\,
	\mathrm{d}\tau.
$

Many mechanical systems have multiple asymptotically stable positions, and are typically separable: the state $\bm{x}$ splits into generalized coordinates $\bm{q}\in \mathbb{R}^{n_p}$ and their conjugate momenta $\bm{p}\in \mathbb{R}^{n_p}$, i.e. $\bm{x} = [\bm{q}^\top, \bm{p}^\top]^\top$.
To encode both properties in the ms-PHNN, we construct the Hamiltonian as
$
	\mathrm{H}_{\bm{\theta}}(\bm{x}) = T_{\bm{\theta}}(\bm{p}) + V_{\bm{\theta}}(\bm{q})
$.
The multiple stable positions $\bar{\bm{q}}_k$ are encoded in the potential energy using the product approach $V_{\bm{\theta}} = \prod_k D_{\bm{\theta},k}(\bm{q})$, so that $V_{\bm{\theta}}$ attains its minimum at each $\bar{\bm{q}}_k$, while the kinetic energy $T_{\bm{\theta}}$ is represented using a single-equlibrium s-PHNN as in~\cite{roth2025stable}, with its unique minimum at some $\bar{\bm{p}}^{\star}$.
Since both summands are non-negative, the Hamiltonian attains its minima at $\bar{\bm{x}}_k = [\bar{\bm{q}}_k^\top, \bar{\bm{p}}^{\star\top}]^\top$, which therefore become equilibria of the ms-PHNN.
This approach further helps the network converge faster, as the separable structure of the Hamiltonian is modeled explicitly rather than learned.

\subsection{Stability Analysis}
\label{sec:stability_analysis}

We now state our main stability guarantees for the proposed ms-PHNN, the proofs of which are given in \cref{sec:appendix_stability}. 
Throughout, we assume the following:
\begin{restatable}[Minimum-energy equilibria satisfy second-order sufficient conditions]{assumption}{AssBregmanSecondOrder}
	\label{ass:bregmann_second_order_sufficient}
	Consider the ms-PHNN~\eqref{eq:phnn},\eqref{eq:hamiltonian}.
	Every Bregman divergence $D_{\bm{\theta},k}$, $k \in \{1, \dots, m\}$, satisfies at its respective minimizer $\bar{\bm{x}}_k \in \bar{\mathbb{X}}$ the second-order sufficient conditions:
	$
		\nabla_{\bm{x}}^2 D_{\bm{\theta},k}(\bar{\bm{x}}_k) \succ \bm{0}.
	$
\end{restatable}
This positive definiteness can be verified directly, or enforced by construction by adding a strongly convex regularization term $\varepsilon \lVert \bm{x} - \bar{\bm{x}}_k \rVert^2$, $\varepsilon > 0$, to the Bregman divergence~\eqref{eq:hamiltonian}, following~\cite{roth2025stable}.
Under this assumption, each minimum-energy equilibrium of the ms-PHNN is locally stable, and asymptotically stable under strict dissipation.
\begin{restatable}[Local asymptotic stability of the minimum-energy equilibria]{theorem}{ThmWellStability}\label{thm:well_stability}
	Consider the ms-PHNN~\eqref{eq:phnn},\eqref{eq:hamiltonian} and $\bm{R}_{\bm{\theta}}(\bm{x}) \succeq \bm{0}$, and let \cref{ass:bregmann_second_order_sufficient} hold.
	In addition, let $\bm{J}_{\bm{\theta}}(\bm{x}) - \bm{R}_{\bm{\theta}}(\bm{x})$ be non-singular for all $\bm{x} \in \mathbb{R}^{n_x}$.
	Then, for each minimum-energy equilibrium $\bar{\bm{x}}_i \in \bar{\mathbb{X}}$, $\mathrm{H}_{\bm{\theta}}$ is a suitable Lyapunov function for showing local stability of $\bar{\bm{x}}_i$, and therefore $\bar{\bm{x}}_i$ is a stable equilibrium.
	Moreover, if $\bm{R}_{\bm{\theta}}(\bm{x}) \succ \bm{0}$ for all $\bm{x}\in \mathbb{R}^{n_x}$, then the equilibrium is locally asymptotically stable.
\end{restatable}

Since the product Hamiltonian under~\Cref{ass:bregmann_second_order_sufficient} has $m \geq 2$ strict minimizers, which are stable equilibria of the ms-PHNN, the mountain pass theorem guarantees the existence of at least a third critical point, which is itself an equilibrium of the ms-PHNN (cf.\ \cref{thm:existance_off_well_equlibia}).
We collect these other equilibria into the set $\mathbb{X}_s = \left\{\bm{x} \in \mathbb{R}^{n_x} \middle| \mathrm{f}_{\mathrm{PH}}(\bm{x}) = \bm{0}\right\} \setminus \bar{\mathbb{X}} \neq \emptyset$, which must lie in a bounded region of the state space (\cref{thm:bounded_off_well_equilibria}).
The existence of these additional equilibria prevents a global convergence guarantee to the minimum-energy equilibria without further assumptions on their stability properties.

\begin{restatable}[Other equilibria are hyperbolic with unstable linearization]{assumption}{AssPhnnEquilibria}
	\label{ass:phnn_equilibria}
	The set $\mathbb{X}_s$ of other equilibria is finite.
	Every $\bm{x}_s \in \mathbb{X}_s$ is hyperbolic, i.e., the corresponding linearization of the ms-PHNN~\eqref{eq:phnn},\eqref{eq:hamiltonian}, $\bm{A}(\bm{x}_s) = \nabla_{\bm{x}} \mathrm{f}_{\mathrm{PH}}(\bm{x}_s)$, has eigenvalues $\lambda_1, \dots, \lambda_{n_x} \in \mathbb{C}$ satisfying $\operatorname{Re}\lambda_i \neq 0$ for all $i$, with at least one unstable direction, $\max_i \operatorname{Re}\lambda_i > 0$.
\end{restatable}

Under this assumption, the stable manifold theorem implies that the stable manifold of every other equilibrium has dimension $\dim \leq n_x - 1$.
Since $\mathbb{X}_s$ is finite, the union of the regions of attraction of all other equilibria, $
	\mathcal{W}^s(\mathbb{X}_s) := \left\{ \bm{x}_0 \in \mathbb{R}^{n_x} \;\middle|\; \lim_{t\to\infty} \bm{x}(t; \bm{x}_0) = \bm{x}_s \in \mathbb{X}_s \right\},
$
has Lebesgue measure zero (cf.\ \cref{thm:region_attraction_off_well}).
Hence, we are able to show that the energy dissipation of the ms-PHNN forces almost every trajectory to converge to one of the target minimum-energy equilibria:\footnote{For the separable canonical ms-PHNN, the same guarantee holds under the weaker requirement that the dissipation acts on the momenta only, $\bm{R}_{\bm{p},\bm{\theta}}(\bm{x}) \succ \bm{0}$ (\cref{thm:almost_global_stability_seperable}).}

\begin{restatable}[Almost global asymptotic stability of the minimum-energy equilibria]{theorem}{ThmAlmostGlobalStability}
	\label{thm:almost_global_stability}
	Consider the ms-PHNN~\eqref{eq:phnn},\eqref{eq:hamiltonian}.
	Let \cref{ass:bregmann_second_order_sufficient} and \cref{ass:phnn_equilibria} hold, and suppose that the dissipation satisfies $\bm{R}_{\bm{\theta}}(\bm{x}) \succ 0$ for all $\bm{x} \in \mathbb{R}^{n_x}$.
	Then for all $\bm{x}_0 \in \mathbb{R}^{n_x} \setminus \mathcal{W}^s(\mathbb{X}_s)$, the solution converges to one minimum-energy equilibrium
		$\lim_{t\to\infty} \bm{x}(t; \bm{x}_0) = \bar{\bm{x}}_i \in \bar{\mathbb{X}}
	$ 
	and the region of attraction $\mathcal{W}^s(\mathbb{X}_s)$ of the other equilibria has Lebesgue measure zero.
\end{restatable}

\section{Experiments}
\label{sec:experiments}
\providecommand{\msdpm}[2]{\ensuremath{#1{\scriptstyle\,\pm\,#2}}}

\begin{table}[t]
  \caption{Derivative RMSE as mean $\pm$ std over 10 seeds,
  and the optimiser-step speedup over NODE final loss. Bold marks the best model. "---" means the model did not reach the reference loss.}
  \label{tab:deriv_rmse}
  \centering
  \small
  \setlength{\tabcolsep}{8.1pt}
  \begin{tabular}{llcccc}
    \toprule
    System & Model & train & validation & test & speedup \\
    \midrule
    \multirow{3}{*}{Duffing}
      & NODE     & \msdpm{0.0033}{0.0009} & \msdpm{0.0049}{0.0009} & \msdpm{0.0039}{0.0008} & $1\times$ \\
      & s-PHNN & \msdpm{0.0850}{0.0165} & \msdpm{0.0888}{0.0179} & \msdpm{0.0883}{0.0187} & --- \\
      & ms-PHNN   & \msdpm{\mathbf{0.0008}}{0.0007} & \msdpm{\mathbf{0.0010}}{0.0007} & \msdpm{\mathbf{0.0009}}{0.0007} & \msdpm{{8.5}}{2.1} \\
    \midrule
    \multirow{3}{*}{Asym. Duffing}
      & NODE     & \msdpm{0.0050}{0.0022} & \msdpm{0.0084}{0.0025} & \msdpm{0.0065}{0.0028} & $1\times$ \\
      & s-PHNN & \msdpm{0.0607}{0.0198} & \msdpm{0.0649}{0.0232} & \msdpm{0.0640}{0.0217} & --- \\
      & ms-PHNN   & \msdpm{\mathbf{0.0031}}{0.0015} & \msdpm{\mathbf{0.0041}}{0.0020} & \msdpm{\mathbf{0.0034}}{0.0017} & \msdpm{2.7}{0.6} \\
    \midrule
    \multirow{3}{*}{Four-Magnets}
      & NODE     & \msdpm{\mathbf{0.0104}}{0.0015} & \msdpm{\mathbf{0.0134}}{0.0017} & \msdpm{\mathbf{0.0152}}{0.0020} & $1\times$ \\
      & s-PHNN & \msdpm{0.0736}{0.0054} & \msdpm{0.0838}{0.0063} & \msdpm{0.0842}{0.0071} & --- \\
      & ms-PHNN   & \msdpm{0.0152}{0.0043} & \msdpm{0.0193}{0.0065} & \msdpm{0.0248}{0.0104} & \msdpm{1.8}{0.5}\\
    \bottomrule
  \end{tabular}
\end{table}
\paragraph{Setup.}
We consider three port-Hamiltonian systems with multiple isolated equilibria: the \emph{Duffing oscillator} of Section~\ref{sec:counterexample}; an \emph{asymmetric Duffing oscillator} variant with the reflection symmetry broken, resulting in two differently shaped attractors; and a \emph{pendulum with four magnets} \citep{ChristianChaosIT}, whose critical set is formed of four minima, four saddles and one maximum respectively.
The governing equations and data splits of each system are in Appendix~\ref{ap:experimental}.
All models are trained by derivative fitting (Adam, learning rate of $10^{-3}$, batch size of $128$, and $50{,}000$ steps) over $10$ seeds given $n=50$ training trajectories with horizon $T=100$ (Duffing) and $T=25$ (Four-Magnets).
We compare an unconstrained {NODE}, {s-PHNN}~\cite{roth2025stable}, and our {ms-PHNN} with a separable Hamiltonian, while initializing the location of each equilibria around the ground truth with Gaussian noise and using FICNNs with $\operatorname{softplus}$ activation function and two hidden layers of size 16 and 32 for the Duffing systems and Four-Magnets respectively\footnote{Same hyperparameters for the MLP used in NODE.}.

\paragraph{Results.}
As shown in Table \ref{tab:deriv_rmse} the derivative root-mean-square error (RMSE) of the ms-PHNN for the Duffing system consistently outperforms the s-PHNN as our theoretical contributions let expect. On the Duffing and asymmetric Duffing system, our ms-PHNN even outperforms the unconstrained NODE, while having a slightly decreased performance on the harder Four-Magnets task.
Inspecting the learning curves for the Duffing experiment in \cref{fig:curves_dw}, we find a significant speedup in terms of convergence rate of our ms-PHNN with respect to the final loss level of NODE. The matching inductive bias in terms of qualitative characteristics of the target system makes it easier for the model to learn the true dynamics. The speedup is also recorded for the asymmetric Duffing variant and the four-magnets pendulum in Table \ref{tab:deriv_rmse}, still showing a training convergence speedup of 1.8$\times$-2.7$\times$. In Appendix \ref{ap:experimental} we also conduct a numerical verification of our Assumption \ref{ass:phnn_equilibria} asserting that our ms-PHNN indeed locates all target equilibria of the ground truth systems.

\section{Conclusion}
In this work, we present an extension of asymptotically stable port-Hamiltonian neural networks to model systems with multiple stable equilibria. We achieve this via an explicit multiplicative structure in which each component, derived from a shared FICNN, models a single energy minimum. Our theory asserts local stability and with appropriate assumptions on the qualitative behavior of the systems we are able to show almost-global stability. Our experiments on three systems confirm the extended capabilities and with a matched inductive bias, our ms-PHNN converges faster than unconstrained NODEs.

\section*{Acknowledgments}
This work has been partially funded by the Deutsche Forschungsgemeinschaft (DFG, German Research Foundation) via the TRR 391 Spatio-temporal Statistics for the Transition of Energy and Transport (project 520388526), the research unit Active Learning for Systems and Control (ALeSCo) (project 535860958), and project 519323897.

\bibliographystyle{plainnat}
\bibliography{references}
\appendix
\crefalias{section}{appendix}
\section{Notation}
Throughout the paper, we use the following notation:
\begin{itemize}
	\item $\mathbb{R}$ and $\mathbb{C}$ denote the real and complex numbers, and $\setminus$ denotes set difference.
	\item $\succ$ and $\succeq$ denote the positive definite and positive semidefinite order on symmetric matrices.
	\item $\lVert \cdot \rVert$ denotes the Euclidean norm.
	\item $\dim(\cdot)$ denotes the dimension of a set.
	\item $M[i,j]$ denotes the entry in row $i$ and column $j$ of the matrix $M$.
	\item $\operatorname{Re}(\cdot)$ denotes the real part of a complex number, and $\lambda_i(\cdot)$ the $i$-th eigenvalue.
	\item $\nabla_{\bm{x}}$ and $\nabla_{\bm{x}}^2$ denote the gradient and Hessian with respect to $\bm{x}$.
	\item $C^k(\mathcal{X}, \mathcal{Y})$ denotes the set of $k$-times continuously differentiable functions from $\mathcal{X}$ to $\mathcal{Y}$, with $C^\infty$ denoting infinite differentiability.
	\item $o(\cdot)$ denotes the little-o asymptotic notation.
	\item $\mu(\cdot)$ denotes the Lebesgue measure.
\end{itemize}

\section{Stability Analysis}
\label{sec:appendix_stability}

In this section, we present the stability analysis.
We begin by analyzing properties of the Bregman divergences and the resulting product Hamiltonian.
We then prove that any ms-PHNN must also possess other equilibria.
Next, we analyze the local stability of the designed minimum-energy equilibria.
Following this, we show that, under an assumption on the behavior of the other equilibria, the minimum-energy equilibria are almost globally asymptotically stable.
Finally, we consider the separable canonical ms-PHNN, for which the dissipation requirement for global asymptotic stability can be relaxed to dissipation on the momenta only.

\subsection{Properties of the Bregman Divergence and the Product Hamiltonian}
First, we derive properties of the Bregman divergence computed from the FICNN with softplus activation, and of the resulting product Hamiltonian.

To this end, consider the FICNN~\cite{amos2017input} with the strictly convex softplus activation $\sigma(x) = \ln(1 + e^x)$, with no activation on the final layer
\begin{subequations}\label{eq:ficnn}
\begin{align}
z_1 &= \sigma(L_0 x + b_0), \\
z_{j+1} &= \sigma(W_j z_j + L_j x + b_j), \quad j \in \{1, \dots, N-1\}, \\
z_{N+1} &= W_N z_N + L_N x + b_N, \\
f_{\bm{\theta}}(x) &= z_{N+1}.
\end{align}
\end{subequations}
The parameters $\bm{\theta} = (W_{1:N}, L_{0:N}, b_{0:N})$ comprise the entrywise non-negative weight matrices $W_j \geq 0$, the passthrough matrices $L_j$, and the bias vectors $b_j$ of all layers.

\begin{theorem}[Bounded Gradient of the Bregman divergence]
\label{thm:bregman-grad-bounded}
	Let $f_{\bm{\theta}} \colon \mathbb{R}^{n_x} \to \mathbb{R}$ be the fully input convex neural network of the form~\eqref{eq:ficnn} with Bregman divergence
	\begin{equation*}
	    D_{\bm{\theta}}(\bm{x})
	    =
	    f_{\bm{\theta}}(\bm{x})
	    - f_{\bm{\theta}}(\bar{\bm{x}})
	    - \nabla_{\bm{x}} f_{\bm{\theta}}(\bar{\bm{x}})^{\!\top}
	    \left(\bm{x} - \bar{\bm{x}}\right) .
	\end{equation*}
	Then the gradient of the Bregman divergence is bounded, that is, there exists a constant $D < \infty$, depending on the weight matrices of the FICNN, such that $\left\lVert \nabla_{\bm{x}} D_{\bm{\theta}}(\bm{x}) \right\rVert \le D$ for all $\bm{x} \in \mathbb{R}^{n_x}$.
\end{theorem}
\begin{proof}
	Differentiating $D_{\bm{\theta}}(\bm{x})$ with respect to $\bm{x}$ gives
	\begin{equation*}
	    \nabla_{\bm{x}} D_{\bm{\theta}}(\bm{x})
	    =
	    \nabla_{\bm{x}} f_{\bm{\theta}}(\bm{x})
	    - \nabla_{\bm{x}} f_{\bm{\theta}}(\bar{\bm{x}}),
	\end{equation*}
	where the affine term contributes the constant $- \nabla_{\bm{x}} f_{\bm{\theta}}(\bar{\bm{x}})$.
	By \cite[Lemma~6.3]{NYSTROM2023100080}, the first-order partial derivatives of the FICNN $f_{\bm{\theta}}$ with respect to the $k$-th component of $\bm{x}$ ($k \in \{1, \dots, n_x\}$) are bounded by
	\begin{equation*}
	    \left\lvert \frac{\partial f_{\bm{\theta}}(\bm{x})}{\partial x_k} \right\rvert
	    \le
	    M_N[1,k],
	\end{equation*}
	for all $\bm{x} \in \mathbb{R}^{n_x}$.
	The matrices $M_p$ depend on the passthrough matrices and are recursively defined as $M_p = \lvert L_p \rvert + W_p M_{p-1}$ with $M_0 = \lvert L_0 \rvert$, yielding $M_N \in \mathbb{R}^{1 \times n_x}$.
	Hence, each element of the gradient is bounded by the corresponding element of $M_N$, and therefore $\lVert \nabla_{\bm{x}} f_{\bm{\theta}}(\bm{x}) \rVert \le \lVert M_N \rVert =: G$ for all $\bm{x}$.
	Finally, the triangle inequality bounds the gradient of the Bregman divergence by $\left\lVert \nabla_{\bm{x}} D_{\bm{\theta}}(\bm{x}) \right\rVert \le \left\lVert \nabla_{\bm{x}} f_{\bm{\theta}}(\bm{x}) \right\rVert + \left\lVert \nabla_{\bm{x}} f_{\bm{\theta}}(\bar{\bm{x}}) \right\rVert \le 2G$, which yields the claim with $D = 2G$.
\end{proof}

\begin{theorem}[Gradient and Hessian of the product Hamiltonian]\label{thm:grad_hess}
	Consider the product Hamiltonian $\mathrm{H}_{\bm{\theta}}(\bm{x})$ from \eqref{eq:hamiltonian}.
	Its gradient is given by
	\begin{equation}
		\label{eq:grad_general}
		\nabla_{\bm{x}} \mathrm{H}_{\bm{\theta}}(\bm{x})
		= \sum_{k=1}^{m} \left( \prod_{\substack{j=1 \\ j \neq k}}^{m} D_{\bm{\theta},j}(\bm{x}) \right) \nabla_{\bm{x}} D_{\bm{\theta},k}(\bm{x}),
	\end{equation}
	and the Hessian is
	\begin{equation}
		\label{eq:hess_general}
		\nabla_{\bm{x}}^2 \mathrm{H}_{\bm{\theta}}(\bm{x})
		= \sum_{k=1}^{m} \left( \prod_{\substack{j=1 \\ j \neq k}}^{m} D_{\bm{\theta},j}(\bm{x}) \right) \nabla_{\bm{x}}^2 D_{\bm{\theta},k}(\bm{x})
		+ \sum_{k=1}^{m} \sum_{\substack{l=1 \\ l \neq k}}^{m} \left( \prod_{\substack{j=1 \\ j \neq k,l}}^{m} D_{\bm{\theta},j}(\bm{x}) \right)
		\nabla_{\bm{x}} D_{\bm{\theta},k}(\bm{x})\, \nabla_{\bm{x}} D_{\bm{\theta},l}(\bm{x})^\top.
	\end{equation}
	At any minimum-energy equilibrium $\bar{\bm{x}}_i \in \bar{\mathbb{X}}$, these simplify to
	\begin{equation}\label{eq:grad_hess_well}
		\nabla_{\bm{x}} \mathrm{H}_{\bm{\theta}}(\bar{\bm{x}}_i) = \bm{0},
		\qquad
		\nabla_{\bm{x}}^2 \mathrm{H}_{\bm{\theta}}(\bar{\bm{x}}_i)
		= \left( \prod_{\substack{j=1 \\ j \neq i}}^{m} D_{\bm{\theta},j}(\bar{\bm{x}}_i) \right) \nabla_{\bm{x}}^2 D_{\bm{\theta},i}(\bar{\bm{x}}_i) \succeq \bm{0}.
	\end{equation}
	Now, let \cref{ass:bregmann_second_order_sufficient} hold.
	Then at each minimum-energy equilibrium $\bar{\bm{x}}_i \in \bar{\mathbb{X}}$ the Hessian of the Hamiltonian is positive definite,
	\begin{equation}\label{eq:hess_well_pd}
		\nabla_{\bm{x}}^2 \mathrm{H}_{\bm{\theta}}(\bar{\bm{x}}_i)
		= \left( \prod_{\substack{j=1 \\ j \neq i}}^{m} D_{\bm{\theta},j}(\bar{\bm{x}}_i) \right) \nabla_{\bm{x}}^2 D_{\bm{\theta},i}(\bar{\bm{x}}_i) \succ \bm{0}.
	\end{equation}
	Outside the minimum-energy equilibria, $\bm{x} \notin \bar{\mathbb{X}}$, the gradient of the Hamiltonian simplifies to
	\begin{equation}\label{eq:grad_off_well}
		\nabla_{\bm{x}} \mathrm{H}_{\bm{\theta}}(\bm{x})
		= \mathrm{H}_{\bm{\theta}}(\bm{x}) \sum_{k=1}^{m} \frac{\nabla_{\bm{x}} D_{\bm{\theta},k}(\bm{x})}{D_{\bm{\theta},k}(\bm{x})}.
	\end{equation}
\end{theorem}
\begin{proof}
	The gradient in~\eqref{eq:grad_general} and Hessian in~\eqref{eq:hess_general} follow directly from applying the product rule to $\mathrm{H}_{\bm{\theta}}(\bm{x}) = \prod_{k=1}^m D_{\bm{\theta},k}(\bm{x})$.
	At any minimum-energy equilibrium $\bar{\bm{x}}_i \in \bar{\mathbb{X}}$, where the corresponding Bregman divergence and its gradient vanish, $D_{\bm{\theta},i}(\bar{\bm{x}}_i) = 0$ and $\nabla_{\bm{x}} D_{\bm{\theta},i}(\bar{\bm{x}}_i) = \bm{0}$, every term in the gradient vanishes, and in the Hessian every term vanishes except the one for $k=i$, yielding~\eqref{eq:grad_hess_well}.
	
	By \cref{ass:bregmann_second_order_sufficient}, $\nabla_{\bm{x}}^2 D_{\bm{\theta},i}(\bar{\bm{x}}_i) \succ \bm{0}$, and since $\bar{\bm{x}}_i \neq \bar{\bm{x}}_j$ for $j \neq i$, each remaining factor satisfies $D_{\bm{\theta},j}(\bar{\bm{x}}_i) > 0$, so the product $\prod_{\substack{j=1 \\ j \neq i}}^{m} D_{\bm{\theta},j}(\bar{\bm{x}}_i)$ is a positive scalar, giving the positive definite Hessian in~\eqref{eq:hess_well_pd}.
	Moreover, by \cref{ass:bregmann_second_order_sufficient}, each Bregman divergence $D_{\bm{\theta},k}$ satisfies the second-order sufficient conditions for a strict local minimum at $\bar{\bm{x}}_k \in \bar{\mathbb{X}}$; by convexity of the Bregman divergence, this local minimizer is in fact the unique global minimizer with $D_{\bm{\theta},k}(\bar{\bm{x}}_k) = 0$, and therefore $D_{\bm{\theta},k}(\bm{x}) > 0$ strictly for all $\bm{x} \notin \bar{\mathbb{X}}$.
	For $\bm{x} \notin \bar{\mathbb{X}}$, we can hence divide each term in~\eqref{eq:grad_general} by $D_{\bm{\theta},k}(\bm{x}) > 0$ and use $\prod_{\substack{j=1 \\ j \neq k}}^{m} D_{\bm{\theta},j}(\bm{x}) = \mathrm{H}_{\bm{\theta}}(\bm{x})/D_{\bm{\theta},k}(\bm{x})$ to obtain~\eqref{eq:grad_off_well}.
\end{proof}

\begin{theorem}[Radial Unboundedness of the Bregman divergence and product Hamiltonian]
	\label{thm:bregmann_radial_unbounded}
	Let each Bregman divergence $D_{\bm{\theta},k} \in C^2(\mathbb{R}^{n_x}, \mathbb{R})$, $k \in \{1, \dots, m\}$, be convex, and let \cref{ass:bregmann_second_order_sufficient} hold.
	Then $D_{\bm{\theta},k}$ is coercive/radially unbounded, i.e., $D_{\bm{\theta},k}(\bm{x}) \to \infty$ as $\lVert \bm{x}\rVert \to \infty$.
	Consequently, the product Hamiltonian $\mathrm{H}_{\bm{\theta}} = \prod_{k=1}^m D_{\bm{\theta},k}$ in~\eqref{eq:hamiltonian} is also coercive/radially unbounded, i.e., $\mathrm{H}_{\bm{\theta}}(\bm{x}) \to \infty$ as $\lVert \bm{x}\rVert \to \infty$.
\end{theorem}
\begin{proof}
	By \cref{ass:bregmann_second_order_sufficient}, $\nabla_{\bm{x}}^2 D_{\bm{\theta},k}(\bar{\bm{x}}_k) \succ 0$, so the minimizer $\bar{\bm{x}}_k$ is a strict local minimizer.
	By convexity of $D_{\bm{\theta},k}$, it is therefore the unique global minimizer, and hence $D_{\bm{\theta},k}(\bm{x}) > 0$ for all $\bm{x} \neq \bar{\bm{x}}_k$.
	Hence, consider the corresponding zero sublevel set, where $D_{\bm{\theta},k}(\bar{\bm{x}}_k)=0$, given by $\left\{ \bm{x}\in\mathbb{R}^{n_x} \mid D_{\bm{\theta},k}(\bm{x})\leq 0 \right\} = \{\bar{\bm{x}}_k\}$, which is nonempty and bounded.

	Since $D_{\bm{\theta},k}$ is continuous and convex, it is a closed proper convex function.
	By~\cite[Corollary~8.7.1]{rockafellar1970convex}, boundedness and nonemptiness of one sublevel set imply boundedness of every sublevel set.
	Therefore, $D_{\bm{\theta},k}$ is coercive/radially unbounded, with $D_{\bm{\theta},k}(\bm{x}) \to \infty$ as $\lVert \bm{x} \rVert \to \infty$.
	
	Finally, we consider the product Hamiltonian: each Bregman divergence is radially unbounded, so whenever $\lVert\bm{x}\rVert\to\infty$ then $D_{\bm{\theta},k}(\bm{x})\to \infty$ for every $k = 1,...,m$.
	Hence the product Hamiltonian must satisfy $\mathrm{H}_{\bm{\theta}} = \prod_{k=1}^m D_{\bm{\theta},k} \to \infty$ as $\lVert\bm{x}\rVert \to \infty$, so it is coercive/radially unbounded.
\end{proof}

\begin{theorem}[Boundedness of Hamilton critical points]\label{thm:bounded_equlibia}
	Consider the Hamiltonian~\eqref{eq:hamiltonian}, where each factor $D_{\bm{\theta},k}$ is the Bregman divergence generated from an FICNN~\eqref{eq:ficnn}.
	There exists a compact region $\Omega_r = \left\{ \bm{x} \in \mathbb{R}^{n_x} \middle| \lVert \bm{x}\rVert \leq r \right\}$ for some $r< \infty$, such that all equilibria $\nabla_{\bm{x}} \mathrm{H}_{\bm{\theta}}(\bm{x}) = \bm{0}$ are contained in this region, i.e.,
	\begin{equation*}
		\left\{ \bm{x} \in \mathbb{R}^{n_x} \mid
		\nabla_{\bm{x}} \mathrm{H}_{\bm{\theta}}(\bm{x}) = \bm{0} \right\}
		\subseteq \Omega_r .
	\end{equation*}
\end{theorem}
\begin{proof}
	Let $\bm{x}^{\mathrm{c}} = \tfrac{1}{m}\sum_{k=1}^m \bar{\bm{x}}_k$ denote the centroid of the
	minimum-energy equilibria.
	For the proof, we consider the inner product of $\nabla_{\bm{x}} \mathrm{H}_{\bm{\theta}}(\bm{x}) (\bm{x}-\bm{x}^{\mathrm{c}})$, which needs to be zero for all equilibria $\bm{x}$ with $\nabla_{\bm{x}} \mathrm{H}_{\bm{\theta}}(\bm{x}) = \bm{0}$.

	For $\bm{x} \notin \bar{\mathbb{X}}$, the gradient of the Hamiltonian is given by
	\begin{equation*}
		\nabla_{\bm{x}} \mathrm{H}_{\bm{\theta}}(\bm{x}) =  \mathrm{H}_{\bm{\theta}}(\bm{x}) \sum_{k=1}^m \frac{\nabla_{\bm{x}} D_{\bm{\theta},k}(\bm{x})}{D_{\bm{\theta},k}(\bm{x})}
	\end{equation*}
	Next, we take the inner product of this gradient and $\bm{x} - \bm{x}^{\mathrm{c}}$ yields
	\begin{equation*}
		\nabla_{\bm{x}} \mathrm{H}_{\bm{\theta}}(\bm{x})^\top (\bm{x} - \bm{x}^{\mathrm{c}})
		= \mathrm{H}_{\bm{\theta}}(\bm{x})
		\sum_{k=1}^m \frac{\nabla_{\bm{x}} D_{\bm{\theta},k}(\bm{x})^\top (\bm{x} - \bm{x}^{\mathrm{c}})}{D_{\bm{\theta},k}(\bm{x})}.
	\end{equation*}
	Splitting $\bm{x} - \bm{x}^{\mathrm{c}} = (\bm{x} - \bar{\bm{x}}_k) + (\bar{\bm{x}}_k - \bm{x}^{\mathrm{c}})$
	in each numerator yields
	\begin{equation*}
		\nabla_{\bm{x}} \mathrm{H}_{\bm{\theta}}(\bm{x})^\top (\bm{x} - \bm{x}^{\mathrm{c}})
		= \mathrm{H}_{\bm{\theta}}(\bm{x})
		\sum_{k=1}^m \frac{
			\nabla_{\bm{x}} D_{\bm{\theta},k}(\bm{x})^\top (\bm{x} - \bar{\bm{x}}_k)
			+ \nabla_{\bm{x}} D_{\bm{\theta},k}(\bm{x})^\top (\bar{\bm{x}}_k - \bm{x}^{\mathrm{c}})
		}{D_{\bm{\theta},k}(\bm{x})}.
	\end{equation*}
	The first-order characterization of convexity~\cite[Section~3.1.3]{vandenberghe2004convex} is $f(\bm{y})\geq f(\bm{x})+\nabla f(\bm{x})^\top(\bm{y}-\bm{x})$.
	Applying this to $D_{\bm{\theta},k}$ with $\bm{y}=\bar{\bm{x}}_k$ and using $D_{\bm{\theta},k}(\bar{\bm{x}}_k)=0$ yields
	$\nabla_{\bm{x}} D_{\bm{\theta},k}(\bm{x})^\top (\bm{x}-\bar{\bm{x}}_k) \geq D_{\bm{\theta},k}(\bm{x})$.
	In addition, for each $k \in \{1, \dots, m\}$, by \cref{thm:bregman-grad-bounded} the gradient of the Bregman divergence $D_{\bm{\theta},k}$ is bounded, $\lVert \nabla_{\bm{x}} D_{\bm{\theta},k}(\bm{x}) \rVert \leq G_k$ for some finite $G_k$, and therefore $\nabla_{\bm{x}} D_{\bm{\theta},k}(\bm{x})^\top (\bar{\bm{x}}_k - \bm{x}^{\mathrm{c}}) \geq - G_k \lVert \bar{\bm{x}}_k - \bm{x}^{\mathrm{c}}\rVert$.
	Hence, 
	\begin{align*}
		\nabla_{\bm{x}} \mathrm{H}_{\bm{\theta}}(\bm{x})^\top
		(\bm{x} - \bm{x}^{\mathrm{c}})
		&\geq
		\mathrm{H}_{\bm{\theta}}(\bm{x})
		\sum_{k=1}^m
		\frac{
			D_{\bm{\theta},k}(\bm{x})
			-
			G_k \lVert \bar{\bm{x}}_k - \bm{x}^{\mathrm{c}}\rVert
		}{
			D_{\bm{\theta},k}(\bm{x})
		}
		\\
		&=
		\mathrm{H}_{\bm{\theta}}(\bm{x})
		\sum_{k=1}^m
		\left(
			1
			-
			\frac{
				G_k \lVert \bar{\bm{x}}_k - \bm{x}^{\mathrm{c}}\rVert
			}{
				D_{\bm{\theta},k}(\bm{x})
			}
		\right).
	\end{align*}
	Each Bregman divergence $D_{\bm{\theta},k}$ is radially unbounded by \cref{thm:bregmann_radial_unbounded}, while $G_k \lVert \bar{\bm{x}}_k - \bm{x}^{\mathrm{c}}\rVert$ is a fixed constant.
	Hence, there exists $r < \infty$ such that $D_{\bm{\theta},k}(\bm{x}) > G_k \lVert \bar{\bm{x}}_k - \bm{x}^{\mathrm{c}}\rVert$ for all $k \in \{1, \dots, m\}$ whenever $\lVert\bm{x}\rVert > r$. 
	We choose $r$ additionally so that $r \geq \lVert \bm{x}^{\mathrm{c}} \rVert$.
	For such $\bm{x}$, every summand in the above bound is strictly positive, and since $\mathrm{H}_{\bm{\theta}}(\bm{x}) > 0$ for $\bm{x} \notin \bar{\mathbb{X}}$, we obtain $\nabla_{\bm{x}} \mathrm{H}_{\bm{\theta}}(\bm{x})^\top (\bm{x} - \bm{x}^{\mathrm{c}}) > 0$.
	Since $\lVert\bm{x}\rVert > r \geq \lVert \bm{x}^{\mathrm{c}}\rVert$ implies $\bm{x} \neq \bm{x}^{\mathrm{c}}$, this rules out $\nabla_{\bm{x}} \mathrm{H}_{\bm{\theta}}(\bm{x}) = \bm{0}$, so $\bm{x}$ cannot be an equilibrium.
	Therefore, every equilibrium satisfies $\lVert\bm{x}\rVert \leq r$, including $\bm{x}^{\mathrm{c}}$ itself, which lies in $\Omega_r$ by construction and is therefore bounded.
\end{proof}
While this result as stated relies on the bounded gradients from~\Cref{thm:bregman-grad-bounded}, it holds analogously when adding the strongly convex regularization term $\varepsilon\lVert \bm{x}-\bar{\bm{x}}_k\rVert^2$ to the Bregman divergence~\citep{roth2025stable}, since $\lVert\nabla_{\bm{x}} D_{\bm{\theta},k}(\bm{x})\rVert/D_{\bm{\theta},k}(\bm{x}) \to 0$ as $\lVert\bm{x}\rVert\to\infty$ continues to hold.

\subsection{Local stability}

We now show that, under Assumption~\ref{ass:bregmann_second_order_sufficient}, each minimum-energy equilibrium of the ms-PHNN is locally stable, and, under strict dissipation, asymptotically stable.

\ThmWellStability*
\begin{proof}
	We base the proof on a standard Lyapunov argument~\cite[Theorem~4.1]{khalil2002nonlinear}.

	By \cref{ass:bregmann_second_order_sufficient} and \cref{thm:grad_hess}, the gradient of the Hamiltonian vanishes at the minimum-energy equilibrium $\bar{\bm{x}}_i \in \bar{\mathbb{X}}$, $\nabla_{\bm{x}} \mathrm{H}_{\bm{\theta}}(\bar{\bm{x}}_i) = \bm{0}$, and the Hessian of the Hamiltonian there is positive definite, $\nabla_{\bm{x}}^2 \mathrm{H}_{\bm{\theta}}(\bar{\bm{x}}_i) \succ \bm{0}$.
	Since $\nabla_{\bm{x}} \mathrm{H}_{\bm{\theta}}(\bar{\bm{x}}_i) = \bm{0}$, the dynamics~\eqref{eq:phnn},\eqref{eq:hamiltonian} give $\dot{\bm{x}} = \left[\bm{J}_{\bm{\theta}}(\bar{\bm{x}}_i) - \bm{R}_{\bm{\theta}}(\bar{\bm{x}}_i)\right]\nabla_{\bm{x}} \mathrm{H}_{\bm{\theta}}(\bar{\bm{x}}_i) = \bm{0}$, so $\bar{\bm{x}}_i$ is indeed an equilibrium of the ms-PHNN~\eqref{eq:phnn}.
	
	In addition, all minimum-energy equilibria $\bar{\bm{x}}_i$ satisfy the second-order sufficient conditions for a strict local minimum of the Hamiltonian with minimal value $\mathrm{H}_{\bm{\theta}}(\bar{\bm{x}}_i) = 0$.
	Due to the continuity of the Hessian of the Hamiltonian and $\nabla_{\bm{x}}^2 \mathrm{H}_{\bm{\theta}}(\bar{\bm{x}}_i) \succ \bm{0}$, there exists a neighborhood $\mathcal{N}_i$ such that $\bar{\bm{x}}_i$ is the only critical point of $\mathrm{H}_{\bm{\theta}}$, and hence, by $\bm{J}_{\bm{\theta}}(\bm{x}) - \bm{R}_{\bm{\theta}}(\bm{x})$ being non-singular the only equilibrium of the ms-PHNN in $\mathcal{N}_i$, with $\mathrm{H}_{\bm{\theta}}(\bm{x}) > 0$ for all $\bm{x} \in \mathcal{N}_i \setminus \left\{ \bar{\bm{x}}_i\right\}$.

	The change in the Hamiltonian along the flow of the ms-PHNN is
	\begin{equation*}
		\dot{\mathrm{H}}_{\bm{\theta}}(\bm{x})
		= -\nabla_{\bm{x}} \mathrm{H}_{\bm{\theta}}(\bm{x})^\top \bm{R}_{\bm{\theta}}(\bm{x})\, \nabla_{\bm{x}} \mathrm{H}_{\bm{\theta}}(\bm{x}) \leq 0
	\end{equation*}
	for all $\bm{x} \in \mathcal{N}_i$.
	Therefore, $\mathrm{H}_{\bm{\theta}}$ is a Lyapunov function in the neighborhood $\mathcal{N}_i$, proving the local stability of the equilibrium $\bar{\bm{x}}_i$ in the neighborhood $\mathcal{N}_i$.

	Moreover, with strict dissipation $\bm{R}_{\bm{\theta}}(\bm{x}) \succ \bm{0}$ for all $\bm{x}\in \mathbb{R}^{n_x}$, the change in the Hamiltonian along the flow of the ms-PHNN becomes 
	$\dot{\mathrm{H}}_{\bm{\theta}}(\bm{x}) < 0$ for all $\bm{x} \in \mathcal{N}_i \setminus \{\bar{\bm{x}}_i\}$, 
	since $\bm{R}_{\bm{\theta}}(\bm{x}) \succ \bm{0}$ and $\nabla_{\bm{x}} \mathrm{H}_{\bm{\theta}}(\bm{x}) \neq \bm{0}$ on $\mathcal{N}_i \setminus \{\bar{\bm{x}}_i\}$. 
	Hence, $\mathrm{H}_{\bm{\theta}}$ is a strict local Lyapunov function and $\bar{\bm{x}}_i$ is asymptotically stable.
\end{proof}

\subsection{Analysis of other equilibria}

The product Hamiltonian possesses multiple strict minima, and therefore, by mathematical necessity, must also possess other critical points, which are therefore equilibria of the ms-PHNN.
We now establish the existence of these equilibria for the proposed ms-PHNN.

\begin{theorem}[Existence of other equilibria]\label{thm:existance_off_well_equlibia}
	Consider the ms-PHNN~\eqref{eq:phnn},\eqref{eq:hamiltonian} with at least minimum-energy equilibria $m \geq 2$.
	Then there exists a third equilibrium with $\nabla_{\bm{x}} \mathrm{H}_{\bm{\theta}}(\bm{x}_s) = \bm{0}$ and $\bm{x}_s \notin \bar{\mathbb{X}} $.

	In addition, this equilibrium $ \bm{x}_s$ is not a local minimizer of the Hamiltonian $\mathrm{H}_{\bm{\theta}}$. 
	That is, every neighborhood of $\bm{x}_s$ contains a point $\bm{x}$ with $\mathrm{H}_{\bm{\theta}}(\bm{x}) < \mathrm{H}_{\bm{\theta}}(\bm{x}_s)$. 
	Because $\bm{x}_s$ is a critical point and not a minimum, the smallest eigenvalue of the (real-valued symmetric) Hessian is non-positive 
    \begin{equation*}
        \lambda_{\min}\!\left( \nabla_{\bm{x}}^2 \mathrm{H}_{\bm{\theta}}(\bm{x}_s) \right) \leq 0.
    \end{equation*}
	
	However, this equilibrium $\bm{x}_s$ is not asymptotically stable, that is, for each $\delta > 0$, there exists $ \bm{x}_0$ with $\lVert  \bm{x}_0 -\bm{x}_s \rVert \leq \delta$, such that $\lim_{t\to \infty}\bm{x}(t; \bm{x}_0)\neq \bm{x}_s$.
	
	Moreover, this equilibrium $\bm{x}_s$ solves the stationarity equation
	\begin{equation}\label{eq:crit_eq}
		\sum_{k=1}^{m} \frac{\nabla_{\bm{x}} D_{\bm{\theta},k}(\bm{x}_s)}{D_{\bm{\theta},k}(\bm{x}_s)} = \bm{0}.
	\end{equation}
\end{theorem}
\begin{proof}
	Since $m\geq 2$, the set of Hamiltonian minimizers $\bar{\mathbb{X}}$ contains at least two distinct minimizers.
	Fix two equilibria $\bar{\bm{x}}_i, \bar{\bm{x}}_j \in \bar{\mathbb{X}}$ with $\bar{\bm{x}}_i \neq \bar{\bm{x}}_j$.
	Then $\mathrm{H}_{\bm{\theta}}(\bar{\bm{x}}_i) = \mathrm{H}_{\bm{\theta}}(\bar{\bm{x}}_j) = 0$ while $\mathrm{H}_{\bm{\theta}}(\bm{x}) > 0$ for $\bm{x} \notin \bar{\mathbb{X}}$, so both equilibria are also strict local minima of the Hamiltonian.
	
	The Hamiltonian $\mathrm{H}_{\bm{\theta}} = \prod_{k=1}^m D_{\bm{\theta},k}$ is radially unbounded/coercive, since it is the product of the radially unbounded/coercive factors $D_{\bm{\theta},k}$; that is, $\mathrm{H}_{\bm{\theta}}(\bm{x}) \to \infty$ as $\lVert\bm{x}\rVert \to \infty$.
	
	Thus $\mathrm{H}_{\bm{\theta}} \in C^1(\mathbb{R}^{n_x}, \mathbb{R})$ is coercive with two distinct strict minima, and the mountain pass theorem~\cite[Theorem 5.2]{jabri2003mountain} yields the existence of a third critical point $\bm{x}_s$, which is not a local minimizer and contains in each neighborhood  a point $\bm{x}$ with $\mathrm{H}_{\bm{\theta}}(\bm{x}) < \mathrm{H}_{\bm{\theta}}(\bm{x}_s)$.
	Since $\bm{x}_s \notin \bar{\mathbb{X}}$ each factor satisfies $D_{\bm{\theta},k}(\bm{x}_s) > 0$, so dividing the stationarity condition $\nabla_{\bm{x}} \mathrm{H}_{\bm{\theta}}(\bm{x}_s) = \sum_{k} \bigl(\prod_{j \neq k} D_{\bm{\theta},j}\bigr)\nabla_{\bm{x}} D_{\bm{\theta},k} = \bm{0}$ from \cref{thm:grad_hess} by $\mathrm{H}_{\bm{\theta}}(\bm{x}_s) > 0$ yields the stationarity condition~\eqref{eq:crit_eq}. 
	
    For the Hessian eigenvalue, we argue by contradiction. 
    Suppose that the Hessian is positive definite, $\nabla_{\bm{x}}^2 \mathrm{H}_{\bm{\theta}}(\bm{x}_s)\succ 0$. 
	Then, by the second-order sufficient condition, $\bm{x}_s$ would be a strict local minimizer of $\mathrm{H}_{\bm{\theta}}$, contradicting the mountain-pass property. 
    Therefore, the Hessian cannot be positive definite, and hence $\lambda_{\min}\!\left(\nabla_{\bm{x}}^2 \mathrm{H}_{\bm{\theta}}(\bm{x}_s)\right)\leq 0$.

	In addition, for this critical point of the Hamiltonian $\nabla_{\bm{x}} \mathrm{H}_{\bm{\theta}}(\bm{x}_s) = \bm{0}$, the dynamics~\eqref{eq:phnn},\eqref{eq:hamiltonian} give $\dot{\bm{x}} = [\bm{J}_{\bm{\theta}}(\bm{x}_s) - \bm{R}_{\bm{\theta}}(\bm{x}_s)]\nabla_{\bm{x}} \mathrm{H}_{\bm{\theta}}(\bm{x}_s) = \bm{0}$, so $\bm{x}_s$ is an equilibrium of~\eqref{eq:phnn}.

	Finally, we prove that $\bm{x}_s$ is not asymptotically stable.
	Each $\delta$-neighborhood of $\bm{x}_s$ contains a point $\bm{x}_0$ with $\lVert  \bm{x}_0 -\bm{x}_s \rVert \leq \delta$ and $\mathrm{H}_{\bm{\theta}}(\bm{x}_0) < \mathrm{H}_{\bm{\theta}}(\bm{x}_s)$.
	Fix the corresponding energy level $\bar{\mathrm{H}} = \mathrm{H}_{\bm{\theta}}(\bm{x}_0)$ and consider the corresponding sublevelset of the Hamiltonian
	\begin{equation*}
		\Omega_{\bar{\mathrm{H}}}
		= \left\{ \bm{x} \in \mathbb{R}^{n_x} 
		\mid
		\mathrm{H}_{\bm{\theta}}(\bm{x}) \leq \bar{\mathrm{H}} 
		\right\}.
	\end{equation*}
	Because the change in energy along the flow of the ms-PHNN is non-positive $\dot{\mathrm{H}}_{\bm{\theta}}(\bm{x})= -\nabla_{\bm{x}} \mathrm{H}_{\bm{\theta}}(\bm{x})^\top \bm{R}_{\bm{\theta}}(\bm{x}) \nabla_{\bm{x}} \mathrm{H}_{\bm{\theta}}(\bm{x}) \leq 0$, the set $\Omega_{\bar{\mathrm{H}}}$ is positively invariant, therefore, $\bm{x}(t; \bm{x}_0) \in \Omega_{\bar{\mathrm{H}}}$ for all $t\geq 0$.
	However, by construction, $\bm{x}_s \notin \Omega_{\bar{\mathrm{H}}}$; in fact, since $\mathrm{H}_{\bm{\theta}}$ is continuous and $\bar{\mathrm{H}} < \mathrm{H}_{\bm{\theta}}(\bm{x}_s)$, the equilibrium $\bm{x}_s$ admits a neighborhood disjoint from $\Omega_{\bar{\mathrm{H}}}$. 
	Therefore, the solutions of the ms-PHNN starting at $\bm{x}_0 \in \Omega_{\bar{\mathrm{H}}}$ remain bounded away from $\bm{x}_s$, i.e., there exists $\rho > 0$ such that $\lVert \bm{x}(t; \bm{x}_0) - \bm{x}_s \rVert \geq \rho$ for all $t \geq 0$, and hence $\lim_{t\to\infty} \bm{x}(t; \bm{x}_0) \neq \bm{x}_s$.
	The existence of such an $\bm{x}_0 $ in each $\delta$-neighborhood of $\bm{x}_s$ contradicts the asymptotic stability of $\bm{x}_s$~\cite[Definition~4.1]{khalil2002nonlinear}.
\end{proof}

Next, we show that the region, where such points can occur is bounded using the preceding result (\cref{thm:bounded_equlibia})

\Cref{thm:existance_off_well_equlibia} guarantees that at least one other equilibrium exists.
However, in general the ms-PHNN can have multiple equilibria which do not correspond to minimum-energy equilibria, which we collect in the set
\begin{equation*}
	\mathbb{X}_s = \left\{\bm{x} \in \mathbb{R}^{n_x} :
	\dot{\bm{x}} = \left[\bm{J}_{\bm{\theta}}(\bm{x}) - \bm{R}_{\bm{\theta}}(\bm{x})\right]\nabla_{\bm{x}} \mathrm{H}_{\bm{\theta}}(\bm{x}) = \bm{0}\right\}
	\setminus \bar{\mathbb{X}},
\end{equation*}
which is nonempty, $\mathbb{X}_s \neq \emptyset$, by \cref{thm:existance_off_well_equlibia}.

Each of these points can be characterized by the linearization of the dynamics at the equilibrium~\cite{haddad2008nonlinear}.
Depending on the eigenvalues of the linearization, the equilibrium is an attractor if all eigenvalues have negative real part (corresponding to a minimum of the Hamiltonian), a repeller if all have positive real part (corresponding to a maximum of the Hamiltonian), and a saddle if the real parts have mixed signs.
In the marginal cases, purely imaginary eigenvalues yield a center, while a zero eigenvalue indicates a higher-order or non-isolated equilibrium.

\begin{theorem}[Boundedness of the other equilibria]\label{thm:bounded_off_well_equilibria}
	Consider the ms-PHNN~\eqref{eq:phnn}, where the Bregman divergences~\eqref{eq:hamiltonian} are generated from an FICNN~\eqref{eq:ficnn}.
	Let $\bm{J}_{\bm{\theta}}(\bm{x}) - \bm{R}_{\bm{\theta}}(\bm{x})$ be non-singular for all $\bm{x} \in \mathbb{R}^{n_x}$.
	Then the set of other equilibria $\mathbb{X}_s$ is bounded, i.e., there exists a compact set $\Omega \subset \mathbb{R}^{n_x}$ such that
	\begin{equation*}
		\mathbb{X}_s \subseteq \Omega.
	\end{equation*}
\end{theorem}
\begin{proof}
	Since $\bm{J}_{\bm{\theta}}(\bm{x}) - \bm{R}_{\bm{\theta}}(\bm{x})$ is non-singular for all $\bm{x} \in \mathbb{R}^{n_x}$, it is injective, such that $\left[\bm{J}_{\bm{\theta}}(\bm{x}) - \bm{R}_{\bm{\theta}}(\bm{x})\right] \nabla_{\bm{x}} \mathrm{H}_{\bm{\theta}}(\bm{x}) = \bm{0}$ holds if and only if $\nabla_{\bm{x}} \mathrm{H}_{\bm{\theta}}(\bm{x}) = \bm{0}$.
	Hence, the set of equilibria of the ms-PHNN~\eqref{eq:phnn} coincides with the set of critical points of the Hamiltonian~\eqref{eq:hamiltonian}, $\left\{ \bm{x} \in \mathbb{R}^{n_x} \mid \nabla_{\bm{x}} \mathrm{H}_{\bm{\theta}}(\bm{x}) = \bm{0} \right\}$.
	By \cref{thm:bounded_equlibia}, this set is contained in the compact region $\Omega_r$.
	Since the other equilibria $\mathbb{X}_s$ are a subset of the equilibria of the ms-PHNN, $\mathbb{X}_s \subseteq \Omega_r$, and hence $\mathbb{X}_s$ is bounded.
\end{proof}

\subsection{Almost global asymptotic stability}
Next, we characterize the region of attraction of the other equilibria under \cref{ass:phnn_equilibria}, which implies that the linearization of the ms-PHNN at every other equilibrium is unstable in at least one direction.
Therefore, each of these critical points corresponds to either a saddle or a maximum of the Hamiltonian.

\begin{theorem}[Region of attraction of the other equilibria]
	\label{thm:region_attraction_off_well}
	Consider a ms-PHNN~\eqref{eq:phnn},\eqref{eq:hamiltonian} and let \cref{ass:phnn_equilibria} hold.
	Then, the region of attraction of the other equilibria
	\begin{equation*}
		\mathcal{W}^s(\mathbb{X}_s) := \left\{ \bm{x}_0 \in \mathbb{R}^{n_x} \;\middle|\; \lim_{t\to\infty} \bm{x}(t; \bm{x}_0) = \bm{x}_s \in \mathbb{X}_s \right\},
	\end{equation*}
	has Lebesgue measure zero.
\end{theorem}
\begin{proof}
	By \cref{ass:phnn_equilibria}, each $\bm{x}_s \in \mathbb{X}_s$ is a hyperbolic equilibrium with at least one eigenvalue having a positive real part.
	Consequently, the stable subspace of the linearization $\bm{A}_s$ has dimension $k \leq n_x -1$.

	Therefore, by the global stable manifold theorem~\citep[Theorem~7.7]{sideris_ordinary_2013}, for each $\bm{x}_s  \in \mathbb{X}_s$ the global stable manifold
	\begin{equation*}
		\mathcal{W}^s(\bm{x}_s) := \left\{ \bm{x}_0 \in \mathbb{R}^{n_x} : \lim_{t \to \infty} \bm{x}(t; \bm{x}_0) = \bm{x}_s \right\},
	\end{equation*}
	is a $C^1$ manifold with dimension equal to that of the stable manifold of the linearization, and therefore $\dim \left(\mathcal{W}^s(\bm{x}_s)\right) = k \leq n_x -1$.
	Therefore, each stable manifold $\mathcal{W}^s(\bm{x}_s)$ for each $\bm{x}_s \in \mathbb{X}_s$ has Lebesgue measure zero, $\mu \left(\mathcal{W}^s(\bm{x}_s)\right) = 0$.
	Finally, since $\mathbb{X}_s$ is finite by \cref{ass:phnn_equilibria},the combined region of attraction for all other equilibria
	\begin{equation*}
		\mathcal{W}^s(\mathbb{X}_s) = \bigcup_{\bm{x}_s \in \mathbb{X}_s} \mathcal{W}^s(\bm{x}_s),
	\end{equation*}
	is a finite union of nullsets and hence has Lebesgue measure zero, $\mu\left(\mathcal{W}^s(\mathbb{X}_s)\right) = 0$.
\end{proof}

With the region of attraction of the other equilibria as a Lebesgue-measure-zero set, we next show that the energy dissipation of the ms-PHNN forces the trajectories starting at almost every initial condition to converge to one of the minimum-energy equilibria.

\ThmAlmostGlobalStability*
\begin{proof}
	We prove that all trajectories converge to $\bar{\mathbb{X}} \cup \mathbb{X}_s$ using LaSalle's invariance principle~\cite[Theorem~4.4]{khalil2002nonlinear} if $\bm{R}_{\bm{\theta}}(\bm{x}) \succ 0$ for all $\bm{x} \in \mathbb{R}^{n_x}$.

	First, we consider the sublevel set of the Hamiltonian
	\begin{equation*}
		\Omega_c = \left\{\bm{x} \in \mathbb{R}^{n_x}  \middle|
		 \mathrm{H}_{\bm{\theta}}(\bm{x}) \leq c \right\}
	\end{equation*}
	By \cref{ass:bregmann_second_order_sufficient} and \cref{thm:bregmann_radial_unbounded}, the Hamiltonian~\eqref{eq:hamiltonian} is radially unbounded, so $\Omega_c$ is a compact set.

	Along any trajectory the dissipation is
	\begin{equation*}
		\dot{\mathrm{H}}_{\bm{\theta}}(\bm{x})
		= \nabla_{\bm{x}} \mathrm{H}_{\bm{\theta}}(\bm{x})^\top \left[\bm{J}_{\bm{\theta}}(\bm{x}) - \bm{R}_{\bm{\theta}}(\bm{x})\right]	\nabla_{\bm{x}} \mathrm{H}_{\bm{\theta}}(\bm{x})
		= - \nabla_{\bm{x}} \mathrm{H}_{\bm{\theta}}(\bm{x})^\top \bm{R}_{\bm{\theta}}(\bm{x}) \nabla_{\bm{x}} \mathrm{H}_{\bm{\theta}}(\bm{x})
		\leq  0.
	\end{equation*}
	so $\mathrm{H}_{\bm{\theta}}(\bm{x})$ is non-increasing and each $\Omega_c$ is positively invariant.
	Consider the set $\mathbb{E}_c = \{\bm{x} \in \Omega_c \mid \dot{\mathrm{H}}_{\bm{\theta}}(\bm{x}) = 0\}$ on which the dissipation vanishes.
	Since $\bm{R}_{\bm{\theta}}(\bm{x}) \succ 0$, the dissipation only vanishes for states where the gradient of the Hamiltonian vanishes, i.e., for $\bm{x} \in \bar{\mathbb{X}} \cup \mathbb{X}_s$.
	The largest invariant set contained in $\mathbb{E}_c$ is hence
	\begin{equation*}
		\mathbb{M}_c = \left\{\bm{x} \in \Omega_c \;\middle|\;
		\nabla_{\bm{x}} \mathrm{H}_{\bm{\theta}}(\bm{x}) = \bm{0}\right\}
		= \left(\mathbb{X}_s \cup \bar{\mathbb{X}}\right) \cap \Omega_c.
	\end{equation*}

	Using LaSalle's invariance principle~\cite[Theorem~4.4]{khalil2002nonlinear}, every trajectory starting in $\Omega_c$ converges to $\mathbb{M}_c$.
	Finally, let $\bm{x}_0 \in \mathbb{R}^{n_x}$ be an arbitrary initial condition, and let $c := \mathrm{H}_{\bm{\theta}}(\bm{x}_0)$ be the corresponding energy level.
	Then $\bm{x}_0 \in \Omega_c$, and since this holds for all $\bm{x}_0 \in \mathbb{R}^{n_x}$, LaSalle's invariance principle yields that every trajectory converges to $\mathbb{X}_s \cup \bar{\mathbb{X}}$, i.e., 
	$\lim_{t \to \infty} \operatorname{dist}\left(\bm{x}(t;\bm{x}_0), \mathbb{X}_s \cup \bar{\mathbb{X}}\right) = 0$.
	Since $\bar{\mathbb{X}}$ is finite by construction and $\mathbb{X}_s$ is finite by \cref{ass:phnn_equilibria}, the set $\bar{\mathbb{X}}\cup\mathbb{X}_s$ is finite and therefore consists of isolated equilibria, and continuity of the trajectory therefore implies that the trajectory converges to a single equilibrium $\bar{\bm{x}} \in \bar{\mathbb{X}} \cup \mathbb{X}_s$, i.e., $\lim_{t \to \infty} \lVert \bm{x}(t;\bm{x}_0) - \bar{\bm{x}} \rVert = 0$.
	In addition, by \cref{thm:region_attraction_off_well}, the region of attraction $\mathcal{W}^s(\mathbb{X}_s)$ of the other equilibria has Lebesgue measure zero, hence, for almost every initial condition $\bm{x}_0 \in \mathbb{R}^{n_x} \setminus \mathcal{W}^s(\mathbb{X}_s)$ the resulting trajectory converges to $\bar{\mathbb{X}}$.
	Finally, since $\bar{\mathbb{X}}$ is finite by construction and $\mathbb{X}_s$ is finite by \cref{ass:phnn_equilibria}, all equilibria are isolated. 
	Hence, for all $\bm{x}_0 \in \mathbb{R}^{n_x} \setminus \mathcal{W}^s(\mathbb{X}_s)$, the trajectory converges to one minimum-energy equilibrium $\bar{\bm{x}}_i \in \bar{\mathbb{X}}$.
\end{proof}

\subsection{Separable Canonical ms-PHNN}
\label{sec:separable_phnn}

Many pH systems, in particular mechanical ones, have a state that splits into generalized coordinates $\bm{q} \in \mathbb{R}^{n_p}$ and conjugate momenta $\bm{p} \in \mathbb{R}^{n_p}$, so that $\bm{x} = [\bm{q}^\top, \bm{p}^\top]^\top \in \mathbb{R}^{n_x}$ with $n_x = 2n_p$.
Accordingly, the Hamiltonian decomposes into a kinetic energy $T_{\bm{\theta}} \in C^{\infty}(\mathbb{R}^{n_p}, \mathbb{R})_{\geq 0}$ and a potential energy $V_{\bm{\theta}} \in C^{\infty}(\mathbb{R}^{n_p}, \mathbb{R}_{\geq 0})$,
\begin{equation*}
	\mathrm{H}_{\bm{\theta}}(\bm{x})
	=
	T_{\bm{\theta}}(\bm{p})
	+
	V_{\bm{\theta}}(\bm{q}),
	\qquad
	V_{\bm{\theta}}(\bm{q}) = \prod_{k=1}^{m} D_{\bm{\theta},k}(\bm{q}).
\end{equation*}
In this, the multiple stable positions are encoded in the potential energy $V_{\bm{\theta}}$ via the product, placing a potential energy minimum at each $\bar{\bm{q}}_k$.
The kinetic energy is an s-PHNN built from a second FICNN $g_{\bm{\theta}} \in C^{\infty}(\mathbb{R}^{n_p}, \mathbb{R})$ as in~\cite{roth2025stable}, with zero momentum as its unique minimizer $\bar{\bm{p}}^\star = \bm{0}$
\begin{equation*}
	T_{\bm{\theta}}(\bm{p})
	=
	g_{\bm{\theta}}(\bm{p})
	- g_{\bm{\theta}}(\bm{0})
	- \nabla_{\bm{p}} g_{\bm{\theta}}(\bm{0})^{\!\top} \bm{p}.
\end{equation*}
Since both summands are nonnegative and vanish only at their respective minimizers, $\mathrm{H}_{\bm{\theta}}$ attains its minima exactly on the set of minimum-energy equilibria
\begin{equation*}
    \bar{\mathbb{X}}
    =
    \left\{
        \begin{bmatrix}
            \bar{\bm{q}}_1 \\
            \bm{0}
        \end{bmatrix},
        \dots,
        \begin{bmatrix}
            \bar{\bm{q}}_m \\
            \bm{0}
        \end{bmatrix}
    \right\}.
\end{equation*}

We further impose a canonical symplectic structure and restrict dissipation to the momenta
\begin{equation*}
	\bm{J} =
	\begin{bmatrix} \bm{0} & \bm{I} \\ -\bm{I} & \bm{0} \end{bmatrix},
	\qquad
	\bm{R}_{\bm{\theta}}(\bm{x}) =
	\begin{bmatrix} \bm{0} & \bm{0} \\ \bm{0} & \bm{R}_{\bm{p},\bm{\theta}}(\bm{x}) \end{bmatrix},
	\qquad
	\bm{R}_{\bm{p},\bm{\theta}}(\bm{x}) = \bm{R}_{\bm{p},\bm{\theta}}(\bm{x})^\top \succ \bm{0}.
\end{equation*}

With this, the full dynamics of the ms-PHNN~\eqref{eq:phnn} are given by
\begin{subequations}
\label{eq:separable_hamiltonian_full}
\begin{equation}
	\dot{\bm{x}} = \begin{bmatrix} \dot{\bm{q}} \\ \dot{\bm{p}} \end{bmatrix}
	=
	\begin{bmatrix}
		\nabla_{\bm{p}} T_{\bm{\theta}}(\bm{p}) \\
		-\nabla_{\bm{q}} V_{\bm{\theta}}(\bm{q}) - \bm{R}_{\bm{p},\bm{\theta}}(\bm{x})\,\nabla_{\bm{p}} T_{\bm{\theta}}(\bm{p})
	\end{bmatrix},
	\qquad
	\bm{x}(0) = \bm{x}_0,
\end{equation}
where
\begin{equation}
	V_{\bm{\theta}}(\bm{q}) = \prod_{k=1}^{m} D_{\bm{\theta},k}(\bm{q}),
	\qquad
	D_{\bm{\theta},k}(\bm{q}) = f_{\bm{\theta},k}(\bm{q}) - f_{\bm{\theta},k}(\bar{\bm{q}}_k) - \nabla_{\bm{q}} f_{\bm{\theta},k}(\bar{\bm{q}}_k)^{\!\top}(\bm{q} - \bar{\bm{q}}_k),
\end{equation}
\begin{equation}
	T_{\bm{\theta}}(\bm{p}) = g_{\bm{\theta}}(\bm{p}) - g_{\bm{\theta}}(\bm{0}) - \nabla_{\bm{p}} g_{\bm{\theta}}(\bm{0})^{\!\top} \bm{p}.
\end{equation}
\end{subequations}

Next, we show that for the separable canonical ms-PHNN the dissipation requirement for asymptotic stability can be relaxed to dissipation on the momentum.

\begin{theorem}[Almost-global stability of minimum-energy equilibria in separable canonical ms-PHNNs]
	\label{thm:almost_global_stability_seperable}
	Consider the separable canonical ms-PHNN~\eqref{eq:separable_hamiltonian_full}
	\begin{equation*}
		\dot{\bm{x}}
		=
		\left[\bm{J}_{\bm{\theta}}(\bm{x}) - \bm{R}_{\bm{\theta}}(\bm{x})\right]
		\nabla_{\bm{x}} \mathrm{H}_{\bm{\theta}}(\bm{x}),
		\qquad
		\bm{J}_{\bm{\theta}}(\bm{x}) = -\bm{J}_{\bm{\theta}}(\bm{x})^\top,
		\quad
		\bm{R}_{\bm{\theta}}(\bm{x}) = \bm{R}_{\bm{\theta}}(\bm{x})^\top \succeq \bm{0}.
	\end{equation*}
	Let \cref{ass:phnn_equilibria} hold and let \cref{ass:bregmann_second_order_sufficient} hold for each Bregman divergence in the product potential energy and the kinetic energy.
	And let the dissipation satisfy $\bm{R}_{\bm{p},\bm{\theta}}(\bm{x}) \succ 0$ for all $\bm{x} \in \mathbb{R}^{n_x}$, then for almost every initial condition $\bm{x}_0 \in \mathbb{R}^{n_x} \setminus \mathcal{W}^s(\mathbb{X}_s)$, the solution converges to one minimum-energy equilibrium
	\begin{equation*}
		\lim_{t\to\infty} \bm{x}(t; \bm{x}_0) = \bar{\bm{x}}_i \in \bar{\mathbb{X}}.
	\end{equation*}
	and the region of attraction $\mathcal{W}^s(\mathbb{X}_s)$ of the other equilibria has Lebesgue measure zero.
\end{theorem}
\begin{proof}
	We prove that all trajectories converge to $\bar{\mathbb{X}} \cup \mathbb{X}_s$ using LaSalle's invariance principle~\cite[Theorem~4.4]{khalil2002nonlinear} if $\bm{R}_{\bm{p},\bm{\theta}}(\bm{x}) \succ 0$ for all $\bm{x} \in \mathbb{R}^{n_x}$.
	
	First, we consider the sublevel set of the Hamiltonian
	\begin{equation*}
		\Omega_c = \left\{\bm{x} \in \mathbb{R}^{n_x}  \middle|
		 \mathrm{H}_{\bm{\theta}}(\bm{x}) \leq c \right\}
	\end{equation*}
	By \cref{ass:bregmann_second_order_sufficient} and \cref{thm:bregmann_radial_unbounded}, both the potential energy $V_{\bm{\theta}}(\bm{q})$ and the kinetic energy $T_{\bm{\theta}}(\bm{p})$ are radially unbounded.
	Therefore, as $\bm{x} = [\bm{q}^\top, \bm{p}^\top]^\top \to \infty$, either $\bm{p} \to \infty$ or $\bm{q} \to \infty$ or both, so the separable Hamiltonian $\mathrm{H}_{\bm{\theta}}(\bm{x})$ is also radially unbounded, and hence $\Omega_c$ is a compact set.
	Along any trajectory the dissipation is
	\begin{align*}
		\dot{\mathrm{H}}_{\bm{\theta}}(\bm{x})
		&= \nabla_{\bm{x}} \mathrm{H}_{\bm{\theta}}(\bm{x})^\top \left[\bm{J}_{\bm{\theta}}(\bm{x}) - \bm{R}_{\bm{\theta}}(\bm{x})\right]	\nabla_{\bm{x}} \mathrm{H}_{\bm{\theta}}(\bm{x})\\
		&= - \nabla_{\bm{x}} \mathrm{H}_{\bm{\theta}}(\bm{x})^\top \bm{R}_{\bm{\theta}}(\bm{x}) \nabla_{\bm{x}} \mathrm{H}_{\bm{\theta}}(\bm{x})\\
		&= -\nabla_{\bm{p}} T_{\bm{\theta}}(\bm{p})^\top \bm{R}_{\bm{p},\bm{\theta}}(\bm{x})\, \nabla_{\bm{p}} T_{\bm{\theta}}(\bm{p})
		\leq  0
	\end{align*}
	so $\mathrm{H}_{\bm{\theta}}(\bm{x})$ is non-increasing and each $\Omega_c$ is positively invariant.
	Consider the set $\mathbb{E}_c = \{\bm{x} \in \Omega_c \mid \dot{\mathrm{H}}_{\bm{\theta}}(\bm{x}) = 0\}$ on which the dissipation vanishes.
	Since $\bm{R}_{\bm{p},\bm{\theta}}(\bm{x}) \succ 0$ for all $\bm{x} \in \mathbb{R}^{n_x}$, the dissipation vanishes if and only if the gradient of the kinetic energy vanishes, i.e., 
	\begin{equation*}
			\mathbb{E}_c 
			= \left\{\bm{x} \in \Omega_c \;\middle|\; \dot{\mathrm{H}}_{\bm{\theta}}(\bm{x}) = 0 \right\}
			= \left\{\bm{x} = \left[\bm{q}^\top, \bm{p}^\top\right]^\top \in \Omega_c \;\middle|\; \nabla_{\bm{p}} T_{\bm{\theta}}(\bm{p}) = \bm{0} \right\}.
	\end{equation*}
	By Assumption~\ref{ass:bregmann_second_order_sufficient} and the convexity of $T_{\bm{\theta}}$, the zero momentum $\bm{p}=\bm{0}$ is the unique global minimizer of $T_{\bm{\theta}}$, and hence, $\nabla_{\bm{p}} T_{\bm{\theta}}(\bm{p}) = \bm{0}$ if and only if $\bm{p} = \bm{0}$.
	Next, we consider the largest invariant set $\mathbb{M}_c$ in $\mathbb{E}_c$.
	For a trajectory to remain in $\mathbb{E}_c$, the momenta must remain constant zero $\bm{p} = \bm{0}$, and therefore $\dot{\bm{p}} = \bm{0}$.
	From the dynamics,
	\begin{equation*}
		\dot{\bm{p}} = -\nabla_{\bm{q}} V_{\bm{\theta}}(\bm{q}) - \bm{R}_{\bm{p},\bm{\theta}}(\bm{x})\nabla_{\bm{p}} T_{\bm{\theta}}(\bm{p}) = \bm{0},
	\end{equation*}
	and since $\nabla_{\bm{p}} T_{\bm{\theta}}(\bm{p}) = \bm{0}$ on $\mathbb{E}_c$, this reduces to $\nabla_{\bm{q}} V_{\bm{\theta}}(\bm{q}) = \bm{0}$.
	The largest invariant set contained in $\mathbb{E}_c$ is hence
	\begin{align*}
		\mathbb{M}_c 
		&= \left\{\bm{x} \in \Omega_c \;\middle|\;
		\nabla_{\bm{x}} \mathrm{H}_{\bm{\theta}}(\bm{x}) = \bm{0}\right\} \\
		&= \left(\mathbb{X}_s \cup \bar{\mathbb{X}}\right) \cap \Omega_c.
	\end{align*}
	Hence, despite the dissipation acting on the momentum only, this largest invariant set $\mathbb{M}_c$ equals the one obtained in the non-separable case of \cref{thm:almost_global_stability}, and hence the remainder of the proof follows that of \cref{thm:almost_global_stability}.
	In addition, by \cref{thm:region_attraction_off_well}, the region of attraction $\mathcal{W}^s(\mathbb{X}_s)$ of the other equilibria has Lebesgue measure zero, hence, for almost every initial condition $\bm{x}_0 \in \mathbb{R}^{n_x} \setminus \mathcal{W}^s(\mathbb{X}_s)$, the trajectory converges to one minimum-energy equilibrium $\bar{\bm{x}}_i \in \bar{\mathbb{X}}$.
\end{proof}

\section{Experimental Details}\label{ap:experimental}
\subsection{Systems}

\paragraph{Duffing Oscillator.}
Defined in terms of Hamiltonian $\rm{H}(q,p) = \tfrac12 p^2 + \tfrac14(q^2-1)^2$ with dampening $\bm{R}=\operatorname{diag}(0,0.1)$ and canonical $\bm{J}=[[0,-1],[-1,0]]$.
This system has two minimum-energy equilibria at $(q,p)=[\pm 1, 0]^T$ and a saddle at $(q,p)=[0,0]^T$.

\paragraph{Asymmetric Duffing.}
A potential of $V(q) = k(q-q_\ell)^2(q-q_r)^2((q-s)^2 + w)$ with $k = 1/80$, $q_\ell = -1.5$, $q_r = 2.0$, $s = 1.0$, $w = 0.5$ and $T(p)=\tfrac12 p^2$ forming the Hamiltonian $\rm{H}(q,p)=V(q)+T(p)$. The pH system comes with dampening $\bm{R}=\operatorname{diag}(0,0.1)$ and canonical $\bm{J}$.
The system breaks the symmetry between $q<0$ and $q>0$ by the strictly positive factor $(q-s)^2+w$: both roots remain global minima of equal depth, but the stiffnesses become
$V''(q_\ell) = 2k(q_\ell - q_r)^2((q_\ell-s)^2+w) = 2.067$ and
$V''(q_r) = 2k(q_\ell - q_r)^2((q_r-s)^2+w) = 0.459$.

\paragraph{Four-magnet pendulum.}
Inspired by \cite{ChristianChaosIT}, we leverage a quadratic potential in a $\mathbb{R}^4$ dimensional state with two positional coordinates $\bm{q}=(x,y)$ and $\bm{p}$ for the momentum. The potential is given by $V(x,y) = K[(x^2+y^2-d^2)^2 + 4\lambda x^2 y^2]$ with $d = \sqrt2$, $\lambda = 4$, $K = 1/(4d^2)$
making the curvatures at the minimum-energy equilibria independently of $d$, so $\lambda$ alone controls anisotropy and $d$ alone the barrier scale. With $T(p)=\tfrac12 \|p\|_2^2$ forming the Hamiltonian $\rm{H}(q,p)=V(q)+T(p)$ and the pH system with dampening $\bm{R}=\operatorname{diag}(0,0,0.85,0.85)$ and canonical $\bm{J}$.
The critical set is four minima ($V=0$), four saddles at radius $d/\sqrt{1+\lambda}$ on the diagonals, and one maximum at the origin.

\subsection{Data Generation}
\label{app:data}
Every trajectory is a solution of $\dot{\bm{x}} = (\bm{J}-\bm{R})\nabla\mathrm{H}(\bm{x})$ obtained with the Tsit5 (the Tsitouras 5(4) explicit Runge--Kutta pair) solver under a proportional–integral–derivative (PID) step-size controller at \mbox{rtol = atol = $10^{-6}$}.
The ground truth gradient $\nabla V$ is supplied in closed form rather than by automatic differentiation.

Initial conditions use a batch-covering scheme: each draw is uniform over the phase window and rejected only if fewer than $2\%$ of its trajectories don't converge to a minimal-energy equilibrium.
The windows are the box $[-2,2]^2$ (Duffing), the box $q_0\in[-2.4,3.2]$, $p_0\in[-2,2]$ (Asym. Duffing), and the discs $\lVert\bm{q}_0\rVert \leq 1.9475$, $\lVert\bm{p}_0\rVert \leq 2.0976$ (Four-Magnets Pendulum), each chosen so the basins are sampled to a comparable potential energy of about $2.2$. Train, validation and test draw three disjoint sets of initial conditions of size $n_{\text{train}} = 50$, $n_{\text{val}} = 50$ and $n_{\text{test}} = 200$; and are integrated on a time grid with 10 steps per unit time interval until $T=100$ for both Duffing and asym. Duffing and with 100 steps per unit time interval until $T=25$ for Four-Magnets due to increased complexity of the system.

\subsection{Numerical Verification of Assumption}
\label{app:audit}

For every trained ms-PHNN instance we locate all critical points of $\mathrm{H}_{\bm{\theta}}$ by damped Newton iteration from a uniform grid of starts covering the sampling window and evaluated $\nabla^2_{\bm{x}}\mathrm{H}_{\bm{\theta}}$ and $\bm{A} = [\bm{J}_{\bm{\theta}}-\bm{R}_{\bm{\theta}}]\nabla^2_{\bm{x}}\mathrm{H}_{\bm{\theta}}$.
\begin{table}[]
  \caption{Numerical verification of Assumption~\ref{ass:phnn_equilibria} over $10$ seeds per system inspecting the eigenspectrum of critical points obtained by a fitted ms-PHNN.}
  \label{tab:audit}
  \centering
  \small
  \setlength{\tabcolsep}{6pt}
  \begin{tabular}{lcccccc}
    \toprule
    & \multicolumn{3}{c}{located\,/\,true} & & & \\
    \cmidrule(lr){2-4}
    System & min & saddle & max & $\lambda_{\min}(\nabla^2\mathcal{H})$ at min & $\lambda_{\min}(\nabla^2\mathcal{H})$ other  & $\min_i \lvert\operatorname{Re}\lambda_i(\bm{A})\rvert$ \\
    \midrule
    Duffing& $2/2$ & $1/1$ & $0/0$ & $[0.028,\, 0.073]$ & $[-0.073,\, -0.028]$ & $\geq 0.471$ \\
    Asym. Duffing     & $2/2$ & $1/1$ & $0/0$ & $[0.005,\, 0.059]$ & $[-0.069,\, -0.006]$ & $\geq 0.320$ \\
    Four-Magnets      & $4/4$ & $4/4$ & $1/1$ & $[0.504,\, 0.902]$ & $[-2.330,\, -0.703]$ & $\geq 0.417$ \\
    \bottomrule
  \end{tabular}
\end{table}

As seen in Table~\ref{tab:audit}, no instance produced another {minimum}: every other critical point had $\lambda_{\min}(\nabla^2\mathcal{H}_{\bm{\theta}}) < 0$ and $\min_i \lvert\operatorname{Re}\lambda_i(\bm{A})\rvert > 0$, as Assumption~\ref{ass:phnn_equilibria} requires and as the Theorem~\ref{thm:almost_global_stability} needs.
\end{document}